\documentclass[letterpaper, 10 pt, conference]{ieeeconf}  

\IEEEoverridecommandlockouts                              

\usepackage{graphicx} 
\usepackage{amsmath} 
\usepackage{amssymb}  
\usepackage{xcolor}
\usepackage{cite}
\usepackage{subcaption}
\usepackage{bm}
\usepackage{hyperref}
\usepackage[ruled,vlined]{algorithm2e}
\usepackage{booktabs}
\usepackage{caption}
\usepackage{cite}
\newtheorem{theorem}{Theorem}
\newtheorem{lemma}{Lemma}

\title{\LARGE \bf
Chance-Constrained Motion Planning using Modeled Distance-to-Collision Functions
}

\author{Jacob J. Johnson$^{1}$, \textit{IEEE Student Member} and Michael C. Yip$^{1}$, \textit{IEEE Senior Member}
\thanks{$^{1}$J. J. Johnson and M. C. Yip are with the Department of Electrical and Computer Engineering,
        University of California San Diego
        {\tt\small{\{jjj025, yip\} @ucsd.edu}}}%

}

\begin{document}

\maketitle
\thispagestyle{empty}
\pagestyle{empty}

\begin{abstract}
This paper introduces Chance Constrained Gaussian Process-Motion Planning (CCGP-MP), a motion planning algorithm for robotic systems under motion and state estimate uncertainties. The paper's key idea is to capture the variations in the distance-to-collision measurements caused by the uncertainty in state estimation techniques using a Gaussian Process (GP) model. We formulate the planning problem as a chance constraint problem and propose a deterministic constraint that uses the modeled distance function to verify the chance-constraints. We apply Simplicial Homology Global Optimization (SHGO) approach to find the global minimum of the deterministic constraint function along the trajectory and use the minimum value to verify the chance-constraints. Under this formulation, we can show that the optimization function is smooth under certain conditions and that SHGO converges to the global minimum. Therefore, CCGP-MP will always guarantee that all points on a planned trajectory satisfy the given chance-constraints. The experiments in this paper show that CCGP-MP can generate paths that reduce collisions and meet optimality criteria under motion and state uncertainties. The implementation of our robot models and path planning algorithm can be found on GitHub\footnote{\href{https://github.com/jacobjj/gp_prob_planning}{https://github.com/jacobjj/gp\_prob\_planning}}.
\end{abstract}

\section{INTRODUCTION}
\noindent In the past few decades, a profusion of work has focused on the motion planning problem for an assortment of tasks such as car navigation around obstacles~\cite{8206458, johnson2020dynamically,li2021mpcmpnet}, constrained robotic manipulation~\cite{4399305, qureshi2020constrained}, and surgical robot automation~\cite{8977357}. However, most motion planning research has focused on demonstrating examples where environments are highly structured, and uncertainties in sensing are overlooked. In reality, robots in the real world will face different sources of uncertainties: 1. errors in system model and sensor measurements, 2. ambiguity in the position of obstacles in the space, and 3. varying physical properties of the environment itself. Motion planning algorithms that consider the collision probability, i.e., \textit{chance constraints}~\cite{4739221}, perform better than previous methods in such unstructured environments~\cite{CCOPP, CCOPP2}.

\begin{figure}
    \centering
    \includegraphics[width=\columnwidth]{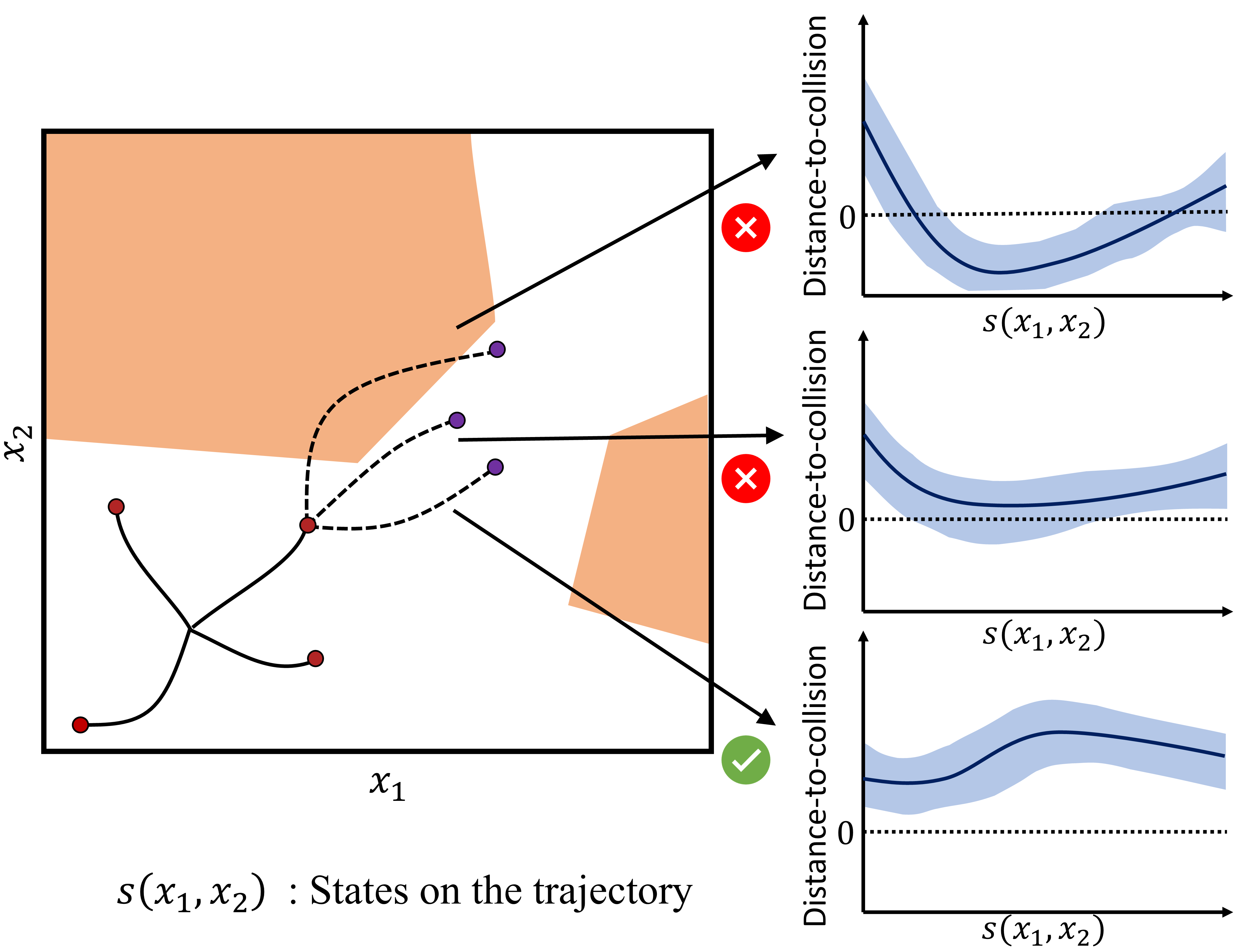}
    \caption{CCGP-MP is a motion planning algorithm for robotic systems under motion and sensor uncertainty which uses a Gaussian Process to model the variations in distance-to-collision. The model verifies user-defined chance constraints for trajectory segments in sampling-based planners.}
    \label{fig:concept_fig}
    \vspace{-1.8em}
\end{figure}

Previous works on planning under uncertainty using chance constraints have not guaranteed that states along a given trajectory are collision-free but rather verify that discrete states satisfy the chance constraints. Furthermore, they bound the obstacles and the robot to make the optimization tractable, making the probabilistic estimates overly conservative, leading to winding trajectories. Finally, estimating collision probabilities utilizing Monte Carlo methods is computationally expensive since the shortest distance from the robot to the obstacle, i.e., \textit{distance-to-collision}, has to be evaluated for a large number of samples.

Ultimately, our goal is to find optimal trajectories that continuously meet chance constraints along an entire trajectory. To this end, we propose the Chance Constrained Gaussian Process-Motion Planning algorithm (CCGP-MP) that addresses those issues. We use a Gaussian Process (GP) to model the distribution of distance-to-collision measures for noisy robotic systems. In turn, we integrate this model with traditional sampling-based planners to generate trajectories that satisfy a given chance constraint. Our formulation ensures both the sampled states in a path satisfy the chance-constraint as well as all the points that lie along the edges connecting the sampled states. Thus, the main contributions of this paper are: 
\begin{enumerate}
    \item Propose a GP model to capture the continuous, probabilistic distribution of functions describing distance-to-collision for a stochastic system
    \item Employ a global optimization technique to verify collision constraints along a given path without discretizing the states along the path.
    \item Generate low-risk paths for systems with motion and sensor noise using sampling-based planners.
\end{enumerate}

\section{RELATED WORKS}
\noindent Many existing planning methods are based on using collision probability for planning under uncertainty~\cite{ lqg_mp, CCOPP, CCOPP2, 8613928, MCMP,MCMP_v2}, while other solutions rely on Markov Decision Process (MDP)~\cite{6280115} or Partially Observable MDPs (POMDPs)~\cite{doi:10.1177/0278364912456319}. MDP and POMDP often need discretization of the state space, and solving an MDP can quickly become computationally intractable for continuous planning domains.
In the following section, we will review a few of the works on estimating collision probability for planning and recent techniques used for distance estimation.

In~\cite{CCOPP}, the authors find a path by formulating the planning problem as an optimization problem where the planned states have to satisfy user-defined chance-constraints while minimizing a cost.  Further development of this algorithm~\cite{CCOPP2} ensured that inter-node trajectories satisfied the chance-constraint but only for obstacles represented as linear functions.  In~\cite{8613928}, the authors propose tighter bounds over ellipsoidal obstacles. For these methods, the number of constraints to solve increases linearly with the number of obstacles, and for higher dimensions, the number of constraints grows exponentially. Using a GP model to capture the distance-to-collision function, we avoid the need to convexify the environment and robot, and irrespective of the environment's complexity, a single equation represents the chance constraints.

Another class of methods estimates the probability of collision along a path by obtaining the robot states' distribution and choosing one with the least likelihood of a collision. Linear-Quadratic Gaussian Motion Planning (LQG-MP)~\cite{lqg_mp} method derives a distribution for the states along a path associated with using a Linear-Quadratic Gaussian (LQG) controller to stabilize the robot. Although this method can obtain a trajectory that reduces the probability of collision, as suggested in~\cite{RRBT}, there is no guarantee that LQG-MP may find a path because of the finite number of paths generated by RRT. In~\cite{6224727}, the authors propose linear constraints on the distribution of states to obtain a tighter collision probability.~\cite{Sun2016} extends the LQG-MP for higher DOF robots, but like LQG-MP, the method does not have a user-defined chance constraint parameter. 
All these methods guarantee safety for discrete states but are intractable to verify safety for all points on a trajectory.

In~\cite{MCMP, MCMP_v2}, the authors use Monte-Carlo simulations to get a more accurate distribution of states and a more precise collision probability estimate. In~\cite{MCMP}, the authors use importance sampling to be more data-efficient in their simulation and reduces the variance of estimates using control variate.~\cite{MCMP_v2} extends this work to non-linear systems for verification of trajectory in an online planning setting. Although these methods estimate collision probability along a path precisely, they rely on meta planning algorithms to generate an initial path and, as such, cannot incorporate additional optimality criteria into the planning problem. CCGP-MP integrates with optimal planners such as RRT* to solve planning problems with optimality criteria.


Many methods are proposed that use geometric sensors and modeling techniques to estimate the distance-to-collision~\cite{ kew2020neural, zhi2021diffco}, and for a brief review, readers can refer to~\cite{das2020stochastic}, but none considers the measurement models in unstructured environments. Das and Yip~\cite{das2020stochastic} proposed one of the first techniques that included uncertainties to the distance measure model. The authors added Gaussian noise to the distance measure but did not consider the effect of state estimation uncertainty while modeling the distribution.


\section{CCGP-Motion Planning}
\noindent In this section, we define our problem and the assumption we make and introduce the building blocks of CCGP-MP.

\subsection{Problem Definition}
\noindent Let the state, control, and observation space be defined as $\mathcal{X}\subseteq\mathbb{R}^{n_x}, \mathcal{U}\subseteq \mathbb{R}^{n_u}$, and $\mathcal{Z}\subseteq \mathbb{R}^{n_z}$ respectively. For a given start position ($x_{start}\in\mathcal{X}$) and goal region ($\mathcal{X}_{goal}\subset \mathcal{X}$ ), the objective is to find a trajectory that satisfies a user-defined collision constraint. A trajectory $\Pi$, defined as a sequence of states $\{x_0, x_1, \ldots, x_N \}$, is considered a solution if $x_0 = x_{start}$, $x_n\in\mathcal{X}_{goal}$, and all the points that connect state $x_i$ and $x_{i+1}$ also satisfy the given collision chance constraint. We assume that for planning the states are sampled in a subspace $\mathbb{X} \subset \mathcal{X}$, where the system's velocities are zero. The problem can be further expanded by enforcing optimality criteria to the sequence of states, such as reducing path length. In the subsequent sections, we detail our solution for this problem. 

\subsection{Motion and Observation Model}
\noindent Throughout this paper, we will suppose that we have a nonlinear dynamics and observation model give by:
\begin{subequations}
\label{eqn:motion_obs_model}
\begin{IEEEeqnarray}{lr}
    x_{t+1} =  f(x_t, u_t) + v(m_t)\quad & m_t\sim\mathcal{N}(0,M)
    \label{eqn:motion_model}
    \\
    z_{t+1} =  h(x_t) + n_t & n_t\sim \mathcal{N}(0, N)
\end{IEEEeqnarray}
\end{subequations}
where $x_t,x_{t+1}\in\mathcal{X}$, $u_t\in\mathcal{U}$, $z_t\in\mathcal{Z}$, $m_t$ and $n_t$ are the process noise sampled from a Gaussian distribution with variance $M\in\mathbb{R}^{n_x \times n_x}$ and $N\in\mathbb{R}^{n_z \times n_z}$ respectively, and $v_t$ additive noise to the motion model at time $t$. Standard filter techniques are used to keep track of the state estimate $\hat{x}_t$ of the true state $x_t$. In~\cite{doi:10.1177/0278364913501564}, the authors show that under certain conditions, for a system given by (\ref{eqn:motion_obs_model}), an LQG controller can drive the system to any point $x\in\mathbb{X}$ starting from any Gaussian distribution. The authors also show that the estimated distribution of states converges to a unique deterministic stationary covariance. We use such a controller for trajectory tracking.

\subsection{Gaussian Process Distance Model}
\begin{figure}
    \centering
    \vspace{2mm}
    \includegraphics[width=\columnwidth]{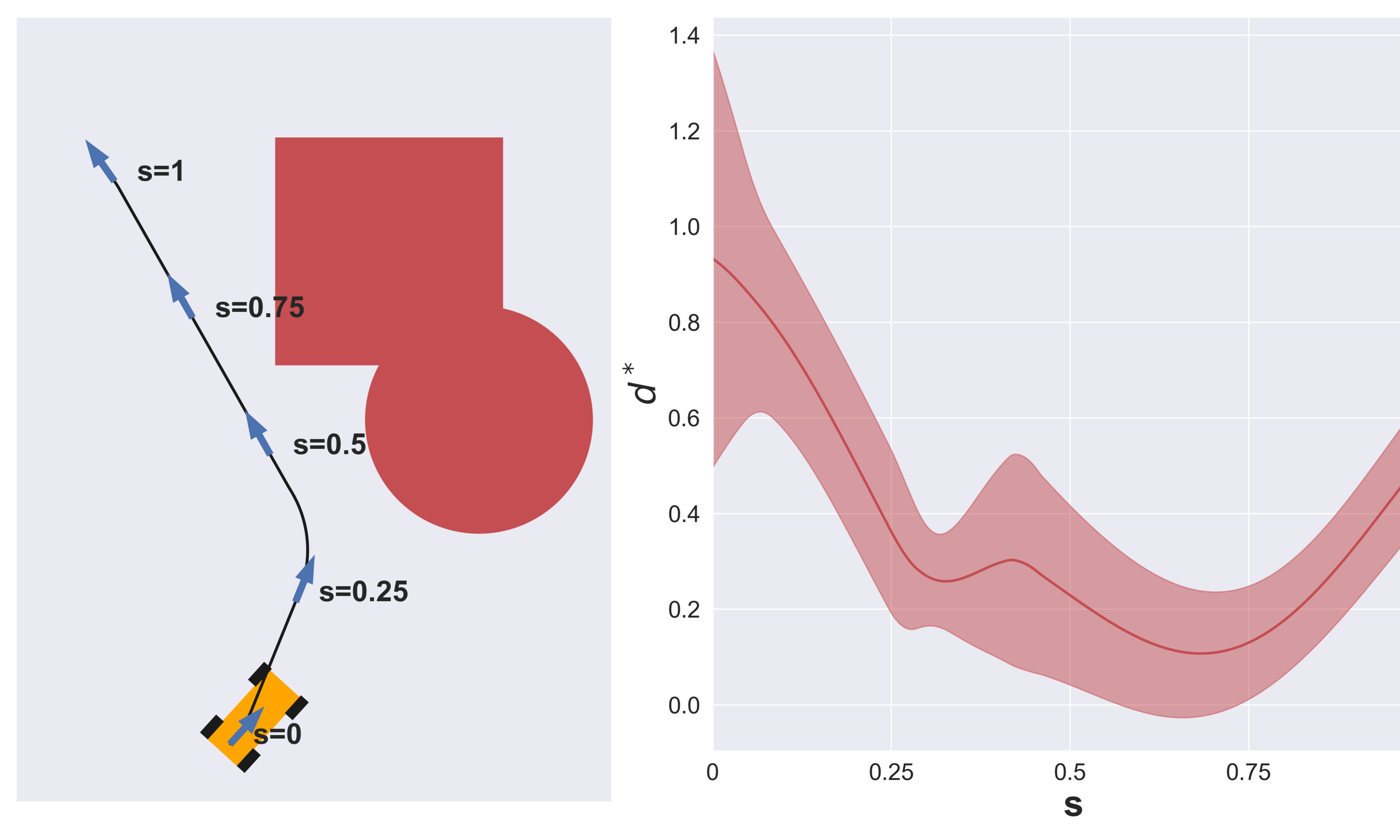}
    \caption{Left: An example of a trajectory in $\mathbb{X}$. It is parameterized using $s$ where $s=0$ and $s=1$ represent the start and end position. Right: The corresponding mean and standard deviation of distance-to-collision ($d^*$), given by (\ref{eqn:distance_distribution}), for points along the trajectory ($s$).} 
    \label{fig:gp_distance_distribution}
    \vspace{-2em}
\end{figure}
\noindent The shortest distance to a collision is modeled probabilistically over the entire state-space using a  GP. Given a model of the environment, a GP is constructed from sampled data by randomly moving the robot model around in the environment model and searching for the distance-to-collision using a geometric method for each estimated state. In a similar fashion to~\cite{das2020stochastic}, the distance-to-collision is evaluated using the Gilbert–Johnson–Keerthi (GJK) method, though any available geometric method can be used in practice. 

Given a prior number of samples $\mathcal{N}$, the set of states $X=\{x_1, x_2, \ldots, x_\mathcal{N}\}$, and the corresponding distance-to-collision from the estimated states, $\bm{d}=\{d_1, d_2, \ldots, d_\mathcal{N}\}$, are used to model the GP. Note that the subscript does not represent a sequence in time, rather a random mix of samples from various trajectories. This data captures the variation in distance-to-collision measure due to uncertainty in the motion and observation model (See Fig. \ref{fig:gp_distance_distribution} for an example of a trajectory in $\mathbb{X}$ and corresponding distance-to-collision distribution). Given data points $X$ and associated distance measure $\bm{d}$, for a state $\bm{x}^*$ the distribution of distance-to-collision, $d^*$, is given by:
\begin{IEEEeqnarray}{rCl}
    d^*\mid X,\bm{d},\bm{x}^* &\sim&  \mathcal{N}(\mathbb{E}[d^*]), \mathbb{V}[d^*])
    \label{eqn:distance_distribution}\\
    \mathbb{E}[d^*] &\triangleq&  \mathcal{K}(\bm{x}^*, X)[\mathcal{K}(X,X) + \sigma^2I]^{-1}\bm{d} \\ 
    \mathbb{V}[d^*] &=& \mathcal{K}(\bm{x}^*,\bm{x}^*)\\
    &&-\mathcal{K}(\bm{x}^*,X)[\mathcal{K}(X,X)+\sigma^2I]^{-1}\mathcal{K}(X,\bm{x}^*) \nonumber
\end{IEEEeqnarray}
where $\mathcal{K}$ is a stationary kernel function and $\sigma^2$ represents the variance of the observed distance measure. For the sake of brevity, we represent the vector $\mathcal{K}(\bm{x}^*, X)$ as $k(\bm{x}^*)$, the matrix $\mathcal{K}(X,X)$ as $K$, and the scalar $\mathcal{K}(\bm{x}^*, \bm{x}^*)$ as $k^*$. The kernel function $\mathcal{K}$ belongs to the class of covariance functions~\cite[Chapter 4]{10.5555/1162254} and is chosen based on the application. A Radial Basis Function (RBF) kernel is popular in most applications, though ~\cite{das2019forward} demonstrates that a forward kinematics kernel provides sparser and more accurate models for robot manipulators. 

\subsection{Chance Constraints}
\noindent The collision probability constraint for the state, $\bm{x}^*\in \mathcal{X}$, of the robot can be represented using the distance measure $d^*$, where 
\begin{equation}
    \mathcal{P}(\bm{x}^* \text{ is in collision})\leq \delta \implies \mathcal{P}(d^*<0) \leq \delta \label{eqn:constraint}
\end{equation}
This probabilistic constraint can be converted to a deterministic constraint as given in~\cite{CCOPP}, where
\begin{IEEEeqnarray}{rCl}
     \mathcal{P}(d^*<0) &\leq &\delta \Longleftrightarrow \frac{\mathbb{E}[d^*]}{\sqrt{2\mathbb{V}[d^*]}}\geq c \label{eqn:CC_eqn}\\
    c &=& \textrm{erf}^{-1}(1-2\delta)
\end{IEEEeqnarray}
and where $\textrm{erf}(z)$ is the Gaussian error function. The ratio of mean to standard deviation in (\ref{eqn:CC_eqn}) is defined as function $g$ given by:
\begin{equation}
     g(\bm{x}^*) \triangleq  \frac{k(\bm{x}^*)^T[K + \sigma^2I]^{-1}\bm{d}}{(2(k^*-k(\bm{x}^*)^T[K+\sigma^2I]^{-1}k(\bm{x}^*)))^\frac12} \label{eqn:mean_by_variance}\
\end{equation}

\subsection{CONNECT Function}
\noindent In sampling-based planners, which include popular approaches such as the many variations of Rapidly Exploring Random Trees (RRTs) and Probabilistic Roadmaps (PRMs), there exists a function that verifies if an edge can connect two nodes by checking if the points on the edge satisfy a set of constraints.
In our work, we are calling these functions \verb|CONNECT| functions. In traditional planners, the \verb|CONNECT| function checks for collision by subsampling the edge and evaluating if each point is collision-free. For probabilistic planners, the \verb|CONNECT| function needs to ensure that all the points on the edge satisfy the chance constraints. To verify if an edge from state $\bm{x_1}$ to state $\bm{x_2}$ meets the given constraints, the condition defined in (\ref{eqn:constraint}) must hold for all points on the edge. We can verify this by checking if the global minimum of (\ref{eqn:mean_by_variance}) satisfies the constraint from (\ref{eqn:CC_eqn}) for the path segment:
\begin{equation}
    \hat{c}\geq c, \qquad \hat{c}= \inf_{\bm{x}^*\in s(\bm{x_1}, \bm{x_2})}g(\bm{x}^*) \label{eqn:global_optim}
\end{equation}
and $s(\bm{x_1}, \bm{x_2})$ represents the trajectory between $\bm{x_1}$ and $\bm{x_2}$.

\subsection{Simplicial Homology Global Optimization}
\noindent To find the global minima of the function $g(\bm{x}^*)$ we use the Simplicial Homology Global Optimization (SHGO) algorithm as proposed in~\cite{SHGOEndres2018}. SHGO is a global optimization technique that exploits the objective function's topography to identify sub-domains where the global minimum may lie. 

The method samples the objective function by a predetermined number of samples and constructs a simplicial complex $\mathcal{H}$. The simplicial complex $\mathcal{H}$ can be conceived as a directed graph, where the vertices represent the value of the objective function at the sampled
points and the directed edges point towards the vertex with a higher objective value.  In  [28], a  vertex $v_i$ is defined as a local minimizer if all the edges connected to $v_i$ are directed away.  A minimizer set, $\mathcal{M}$, is formed with all such local minimizers. st($v_i$) defines a new space called the star of a vertex $v_i$ as the set of points $Q$ such that every simplex containing $Q$ contains $v_i$. The global minimum is found by searching through each sub-domain, st($v_i$), for all $v_i\in\mathcal{M}$. The authors prove that the cardinality of $\mathcal{M}$ remains unchanged with increasing samples, i.e., the number of regions to search for the global minimum does not change with increasing samples. The following theorem guarantees a stationary point in each of these sub-domains:
\begin{theorem}
Given a minimizer $v_i\in\mathcal{M}\subseteq\mathcal{H}$ on the surface of a continuous, Lipschitz smooth objective function $f$ with a compact bounded domain in $\mathbb{R}^n$ and range $\mathcal{R}$, there exist at least one stationary point of $f$ within the domain defined by st$(v_i)$~\cite{SHGOEndres2018}.
\label{theorem:exi_sol}
\end{theorem}
Thus if the objective function is Lipschitz smooth and an adequate number of samples is given, SHGO is able to converge to the global minimum. 

In our situation, to use SHGO to identify the global minimum, we show that (\ref{eqn:mean_by_variance}) is Lipschitz smooth. (\ref{eqn:mean_by_variance}) can be re-written as a composition of two functions, $g(\bm{x}^*) = q\circ k(\bm{x}^*)$. For simplicity, we assume $k^*$ to be 1. First, we show that the function $q$ is Lipschitz continuous.
\begin{lemma}
For $\bm{k}\in [0,1]^n$, $K\succeq0$ and $\sigma >0$, the function $q(\bm{k})$ given by
\begin{equation}
    q(\bm{k}) = \frac{\bm{k}^T(\sigma^2 I+K)^{-1}\bm{d}}{(2(1-\bm{k}^T(\sigma^2 I+K)^{-1}\bm{k}))^{1/2}}
    \label{eqn:function_g}
\end{equation}
satisfies the Lipschitz condition:
\begin{IEEEeqnarray}{rl}
        &\|q(\bm{k_1})-q(\bm{k_2})\| \leq L_q \|\bm{k_1}-\bm{k_2}\|\\
        L_q &=\frac{\|(\sigma^2 I +K)^{-1}\bm{d}\|}{\sqrt{2}}\left( \frac{1}{1-\lambda_{max}n}\right)^{3/2}
\end{IEEEeqnarray}
where $\lambda_{max}$ is the largest eigenvalue of $(I\sigma^2 + K)^{-1}$, and $\bm{k_1}, \bm{k_2}\in[0,1]^n$.
\label{lemma:mean_var}
\end{lemma}
\begin{proof}
For $\bm{k_1}, \bm{k_2} \in [0,1]^n$, we can write $\|q(\bm{k_1})-q(\bm{k_2})\|$ as:
\begin{IEEEeqnarray}{rl}
    \|q(\bm{k_1})-q(\bm{k_2})\| = &\left\| \frac{\bm{k_1}^T(\sigma^2 I+K)^{-1}\bm{d}}{(2(1-\bm{k_1}^T(\sigma^2 I+K)^{-1}\bm{k_1}))^{1/2}} \right.\label{eqn:complex_eqn2}\\
   &\left.-\frac{\bm{k_2}^T(\sigma^2 I+K)^{-1}\bm{d}}{(2(1-\bm{k_2}^T(\sigma^2 I+K)^{-1}\bm{k_2}))^{1/2}}\right\| \nonumber
\end{IEEEeqnarray}
Let $M=(\sigma^2 I+K)^{-1}$. Since $K$ is a Gram matrix of a covariance function, the matrix $M$ is positive definite and symmetric~\cite[Chapter 4]{10.5555/1162254}. Hence we can simplify (\ref{eqn:complex_eqn2}) as:
\begingroup
\allowdisplaybreaks
\begin{align*}
    (\ref{eqn:complex_eqn2}) &\leq \left\| \frac{\bm{k_1}}{(1-\bm{k_1}^TM\bm{k_1})^{1/2}} - \frac{\bm{k_2}}{(1-\bm{k_2}^TM\bm{k_2})^{1/2}} \right\| \frac{\|M\bm{d}\|}{\sqrt{2}} \\
                                 &\tag{from Cauchy-Schwarz inequality}\\
                             &\leq \left \| \bm{k_1}\left(1+\sum_{m=1}^{\infty} \frac{(\bm{k_1}^TM\bm{k_1})^m}{m!}\prod_{i=0}^{m-1}\left(\frac{1}{2}+i\right) \right)\right. \\
                             &\left.\ - \bm{k_2}\left(1+\sum_{m=1}^{\infty} \frac{(\bm{k_2}^TM\bm{k_2})^m}{m!}\prod_{i=0}^{m-1}\left(\frac{1}{2}+i\right) \right) \right \| \frac{\|M\bm{d}\|}{\sqrt{2}}\\
                             \\
                             & \tag{from (\ref{eqn:Taylor_series}) in Appendix} \\
                             &\leq \left\| \bm{k_1} - \bm{k_2}  + \sum_{m=1}^{\infty} \frac{(\bm{k_1}^TM\bm{k_1})^m \bm{k_1}-(\bm{k_2}^TM\bm{k_2})^m\bm{k_2}}{m!}\right.\\
                             &\qquad\qquad\qquad\qquad\qquad\qquad \left.\prod_{i=0}^{m-1}\left(\frac{1}{2}+i\right)\right\|\frac{\|M\bm{d}\|}{\sqrt{2}} \\
                             &\leq \left(\| \bm{k_1}-\bm{k_2}\| +\sum_{m=1}^{\infty} \frac{\|\|\bm{k_1}\|_M^{2m} \bm{k_1}-\|\bm{k_2}\|_M^{2m}\bm{k_2}\| }{m!}\right.\\
                         &\qquad\qquad\qquad\qquad\qquad\qquad \left.\prod_{i=0}^{m-1}\left(\frac{1}{2}+i\right)\right) \frac{\|M\bm{d}\|}{\sqrt{2}}\\
                             &\tag{from triangle inequality} \\
                             &\leq \|\bm{k_1}-\bm{k_2}\|\frac{\|M\bm{d}\|}{\sqrt{2}}\\
                             &\quad\left( 1 + \sum_{m=1}^{\infty} \frac{(1+2m)(\lambda_{max}n)^m}{m!}\prod_{i=0}^{m-1}\left(\frac{1}{2}+i\right)\right)\\
                             &\tag{from Lemma \ref{lemma:norm_over_vec} in Appendix} \\
                             &\leq \|\bm{k_1}-\bm{k_2}\|\frac{\|M\bm{d}\|}{\sqrt{2}}\\
                             &\quad\left( 1 + \sum_{m=1}^{\infty} \frac{(\lambda_{max}n)^m}{m!}\prod_{i=0}^{m-1}\left(\frac{1}{2}+i\right) \right.\\
                             &\quad \qquad\left.+ 2\sum_{m=1}^{\infty} \frac{m(\lambda_{max}n)^m}{m!}\prod_{i=0}^{m-1}\left(\frac{1}{2}+i\right)\right ) \\
                             &\leq \|\bm{k_1}-\bm{k_2}\|\frac{\|M\bm{d}\|}{\sqrt{2}}\left( \frac{1}{(1-\lambda_{max}n)^{1/2}} +\right. \\
                             & \left. \qquad \qquad \qquad \qquad 2\frac{\lambda_{max}n}{2(1-\lambda_{max}n)^{3/2}} \right) \\
                             & \tag{from (\ref{eqn:series_1}) and (\ref{eqn:series_2}) in Appendix} \\
                             & \leq \|\bm{k_1}-\bm{k_2}\|\frac{\|M\bm{d}\|}{\sqrt{2}}\left( \frac{1}{1-\lambda_{max}n}\right)^{3/2}
\end{align*}
\endgroup
\end{proof}
Using Lemma \ref{lemma:mean_var}, we can show that (\ref{eqn:mean_by_variance}) is Lipschitz continuous for a Lipschitz continuous kernel function $k$.
\begin{theorem}
For a Lipschitz continuous kernel $k$, (\ref{eqn:mean_by_variance}) satisfies the Lipschitz condition.
\begin{equation}
    \|q\circ k(\bm{x_1}) - q\circ k(\bm{x_2})\| \leq L_q L_k \|\bm{x_1}-\bm{x_2}\| 
\end{equation}
where $L_k$ is the Lipschitz constant for the kernel function.
\end{theorem}
Thus for planning, the \verb|CONNECT| function within sampling-based planners finds the global minimum of (\ref{eqn:mean_by_variance}) using SHGO and verifies that the user-defined chance constraints are satisfied for the given segment.
\begin{figure*}[t]
    \centering
    \vspace{2mm}
    \includegraphics[width=\textwidth]{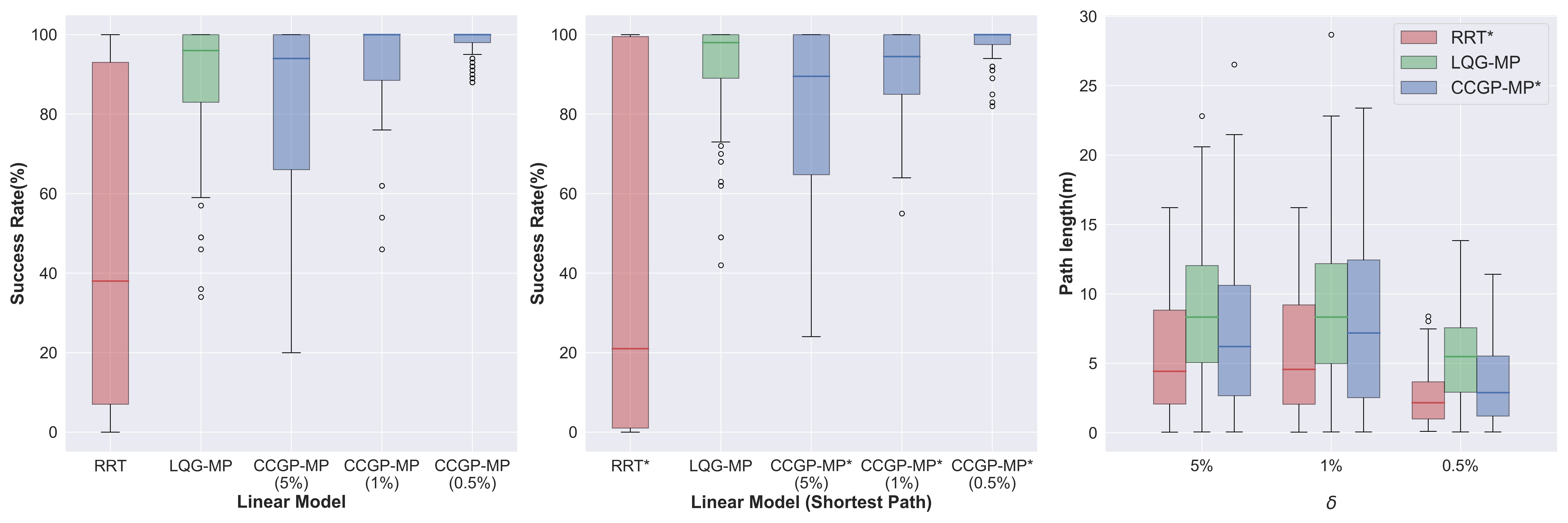}
    \caption{We compare the success rate and path length of the planned paths for the Linear model. Left: The quartile plot compares the success rate for a planning problem without any optimality criteria. Center: The quartile plot compares the success rate for a planning problem with added optimality criteria for reducing path length. Right: The quartile plot compares the length for the same set of start and goal pairs for different $\delta$ values. The plans generated by CCGP-MP and CCGP-MP* are robust to motion and sensor noise, and CCGP-MP* generates shorter paths than LQG-MP.}
    \label{fig:linear_experiments}
    \vspace{-1.5em}
\end{figure*} 
%

\section{EXPERIMENTS AND RESULTS}
\noindent To evaluate the CCGP-MP technique's performance, we tested it on two noisy robot models - a Linear and a Dubins Car model. We explored the planner's performance for 200 random start and goal pairs for different $\delta$ values on randomly generated environments of blocks and circles. Next, to investigate the effects of the increased number of obstacles in the environment on the planning time and accuracy, we evaluated CCGP-MP for the Linear system on 6 randomly generated environments for 10 random start and goal pairs. In addition to these simulated test environments, we assessed the Dubins Car model in a realistic indoor environment taken from the Gibson Environment suite~\cite{xiazamirhe2018gibsonenv}.  
\begin{figure}[h]
    \centering
    \vspace{2mm}
    \includegraphics[width=\columnwidth]{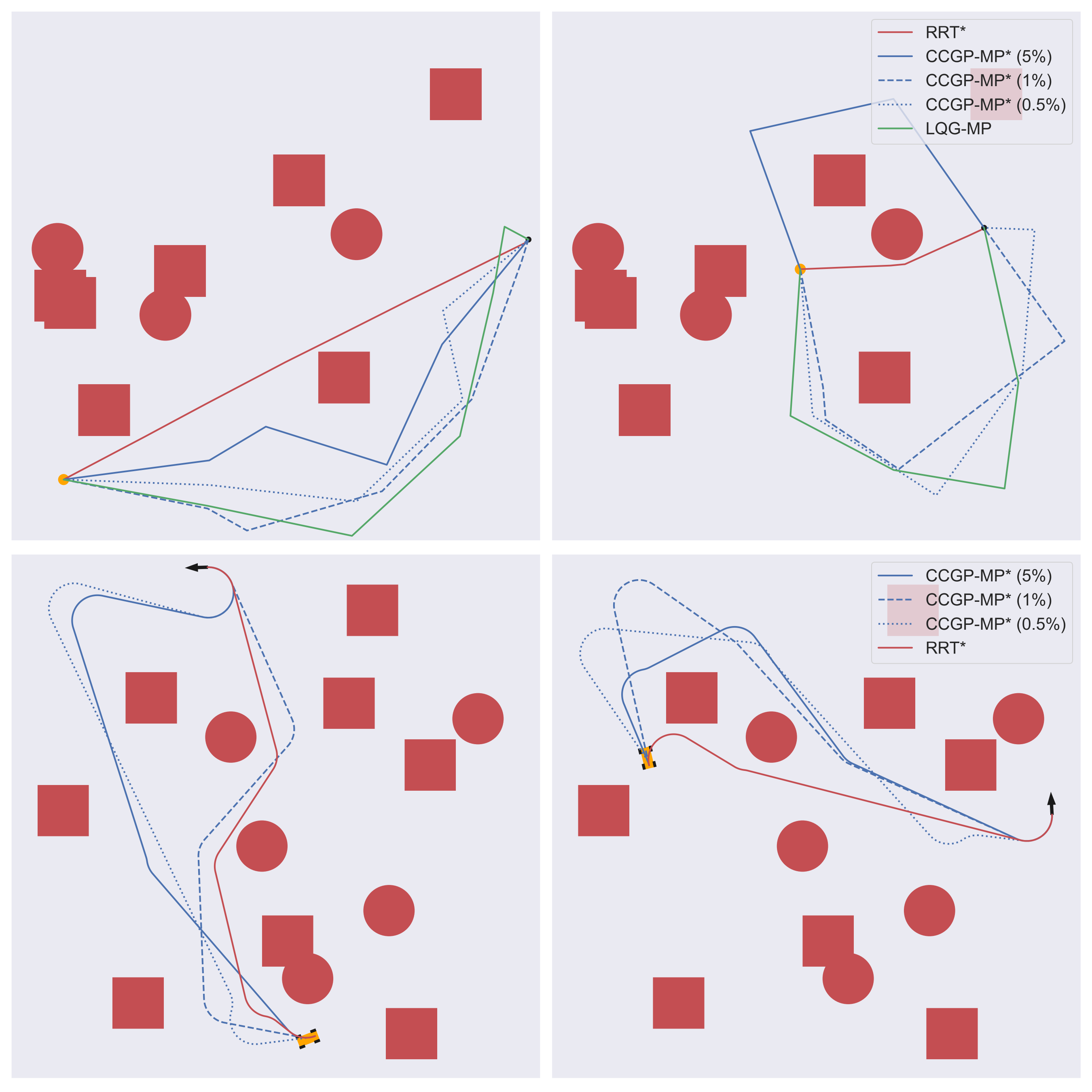}
    \caption{The top row shows the plans generated for different start and goal pairs for the Linear model, while the bottom row does the same for the Dubins Car model. The goals are marked using a black circle for the Linear model and a black arrow for the Dubins Car model. CCGP-MP* deviates from the RRT* planner to satisfy chance constraints.}
    \label{fig:Paths}
    \vspace{-2em}
\end{figure}

In our experiments, we report the $\delta$ values used as percentages, since from our definition in (\ref{eqn:constraint}), it represents the upper bound of the probability measure of a state in a collision. For each robot, 2000 state-distance pairs sampled randomly in the environment, and a RBF kernel are used to define the GP model. To measure the distance to a collision for sampled states, GJK was used over the map. Each planned path's performance was evaluated for 100 trials using a Linear Quadratic Regulator (LQR) based trajectory following controller. A trial was concluded to be successful if the robot could reach the goal region without colliding with any obstacles. We used the Open Motion Planning Library (OMPL)~\cite{ompl} to implement the RRT and RRT* planners. We integrated the \verb|CONNECT| function for both RRT and RRT* planners and called the resulting planners CCGP-MP and CCGP-MP*, respectively. All experiments were written in the Python Programming Language and executed on an AMD Ryzen 2950x CPU with 32GB of RAM. In this section, we provide details of our experiment setup and report our results.

\subsection{Linear Model}

\begin{figure*}
    \centering
    \vspace{2mm}
    \includegraphics[width=2\columnwidth]{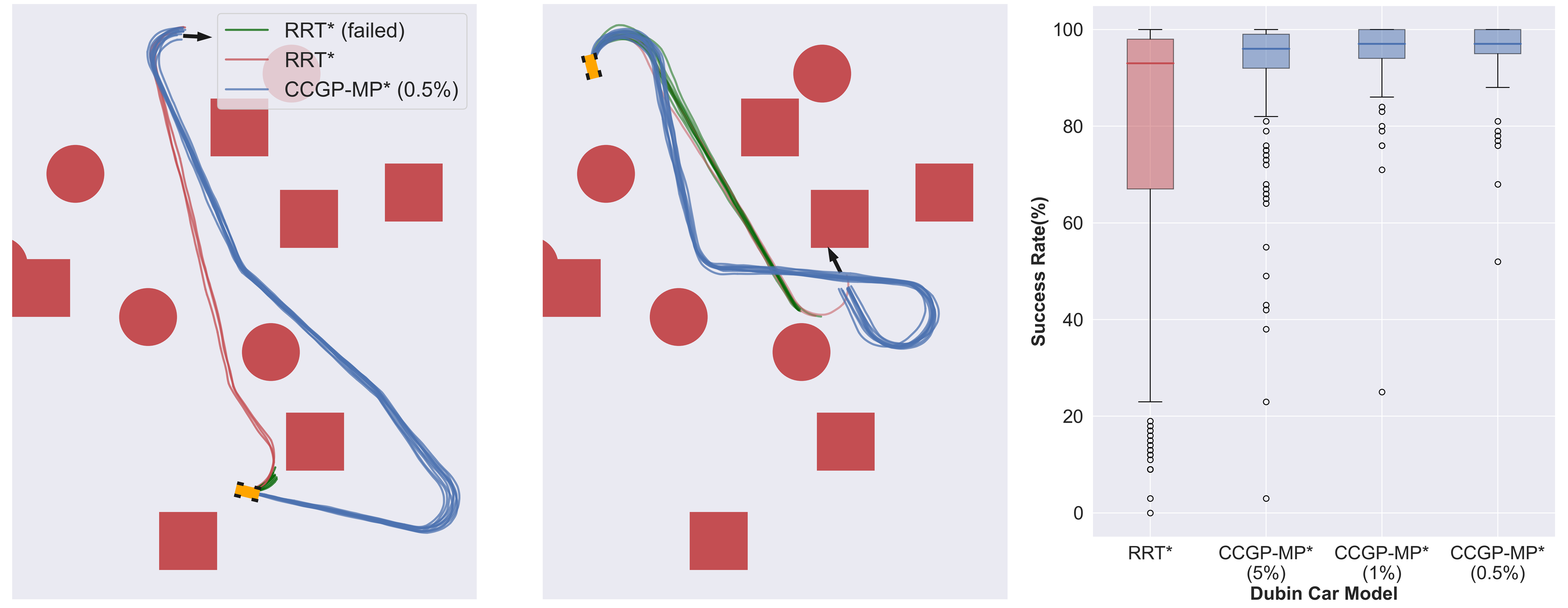}
    \caption{Left and Center: Two examples of path roll-outs for RRT* and CCGP-MP* for a noisy Dubins Car model. The robot collides with the obstacle while executing the path from RRT* (green), while the trajectory generated by CCGP-MP* operates without collision. Right: The quartile plot comparing the success rate of planned paths for random start and goal pairs for a noisy Dubins Car model.}
    \vspace{-2em}
    \label{fig:dubins_ccgp_comb}
\end{figure*}
\noindent The first system we tested was a simple 2D model (see Fig. \ref{fig:Paths} top row). As the robot is symmetric about its base, we plan in the $\mathbb{R}^2$ space. The discrete-time dynamics model of the system is given by:
\begin{subequations}
\begin{align}
    f(\bm{x},\bm{u}) &= \bm{x} + \bm{u} + m, \qquad m\sim N(0, 0.1I) \\
    \bm{z} &= f(\bm{x},\bm{u}) + n \qquad n\sim N(0, 0.01I) 
\end{align}
\end{subequations}
The LQG controller used a Kalman Filter for state estimation and an infinite horizon LQR for generating the control signal.
\subsection{Dubins Model}
\noindent We also tested CCGP-MP on a Dubins Car model whose dynamics is described by:
\begin{subequations}
\begin{align}
f(\bm{x}, u) &= \bm{x} + 
\begin{bmatrix}
\frac{v}{u}(-\sin(\theta)+\sin(\theta+\tau u))\tau \\ \frac{v}{u}(\cos(\theta)-\cos(\theta+\tau u))\tau  \\
u\tau
\end{bmatrix} 
\end{align}
\end{subequations}
where $\tau$ is the time step, and $v$ is the linear velocity of the car. The state $\bm{x} =\begin{bmatrix} x & y & \theta \end{bmatrix}$ and control $u=\omega$ where $(x,y)$ represents the position, $\theta$ the orientation, and $\omega$ the robot's angular velocity. The noisy motion model is implemented as described in~\cite{10.5555/1121596} with linear and angular velocity noise sampled from from $\mathcal{N}(0, 0.1)$, and rotation noise sampled from $\mathcal{N}(0,5^{\circ})$ . The boundary value problem of connecting the sampled state was solved using Dubins curves. An Extended Kalman Filter was used to estimate the robot state, and similar to~\cite{doi:10.1177/0278364913501564}, and a time-varying LQG controller was used to track the trajectory.
\subsection{Experiment Results}
\noindent We compared CCGP-MP against the RRT planner and the Linear Quadratic Gaussian-Motion Planning (LQG-MP) algorithm for the Linear system in a simulated environment. Fig. \ref{fig:linear_experiments} (center) compares the planner's performance with an added optimality criteria of finding the path with the shortest length, and Fig. \ref{fig:linear_experiments} (right) reports the corresponding path length. From Fig. \ref{fig:linear_experiments} (left) and (center), it is evident that considering uncertainties while planning has a significant impact on the success rate of the trajectories. The variance of the success rate for CCGP-MP reduces with decreased $\delta$. The CCGP-MP and CCGP-MP* planner have an equivalent or lower standard deviation of success-rate than LQG-MP for lower thresholds. The better performance for the CCGP methods is because they ensure all intermediate points on a path satisfy the chance-constraint while LQG-MP does not have such guarantees.

Fig. \ref{fig:linear_experiments} (right) reports the length of paths generated by the different planners for the same set of start and goal pairs. The paths generated by CCGP-MP* are shorter compared to the paths generated by LQG-MP. CCGP-MP* is able to generate optimal paths because the underlying planner of CCGP-MP* uses the RRT* algorithm, which is an asymptotically optimal planner~\cite{rrt_star} that makes no assumptions on the \verb|CONNECT| function. Fig. \ref{fig:Paths}, plots the different plans generated by RRT*, LQG-MP, and CCGP-MP* for a random start and goal point. From the image, we may infer that CCGP-MP* deviates from the RRT* plan where the chance constraints are not met, which results in better accuracy for these paths. In comparison, the paths from RRT*, although shorter, would result in multiple failures during execution because of the noisy robot motion. The paths from LQG-MP, on the other hand, although safer, are not optimal.

\begin{figure*}[t]
    \centering
    \begin{subfigure}[c]{1\columnwidth}
        \centering
        \includegraphics[width=0.7\textwidth, angle=90]{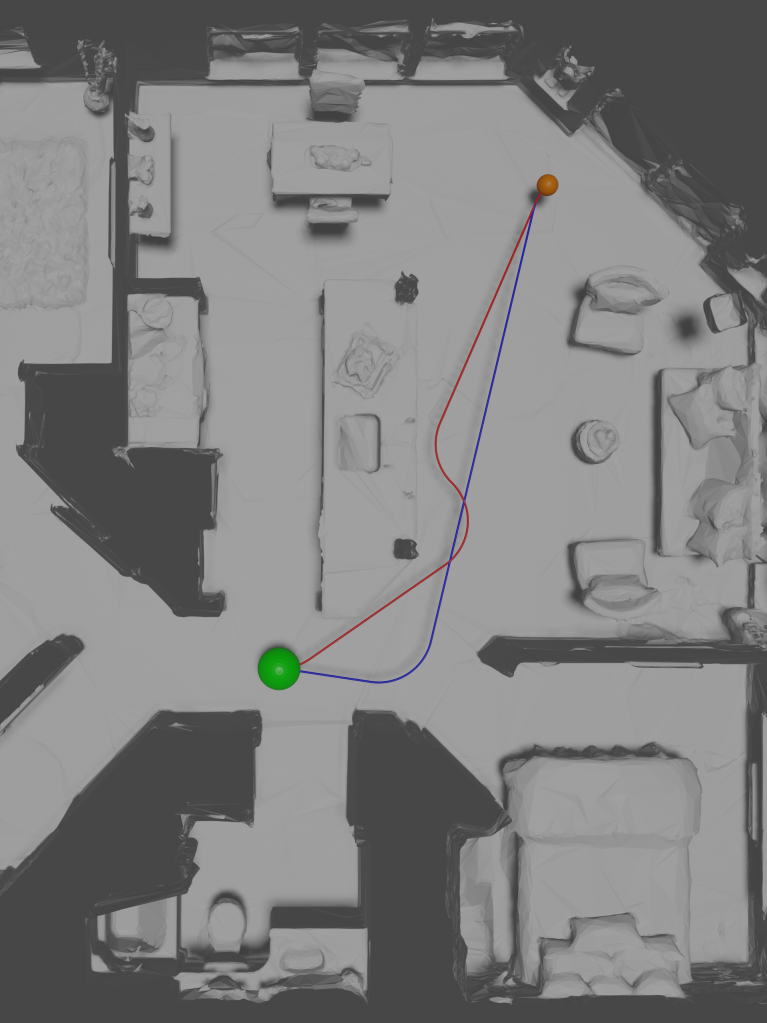}
    \end{subfigure}
    \hfill
    \begin{subfigure}[c]{\columnwidth}
        \centering
        \includegraphics[width=\textwidth]{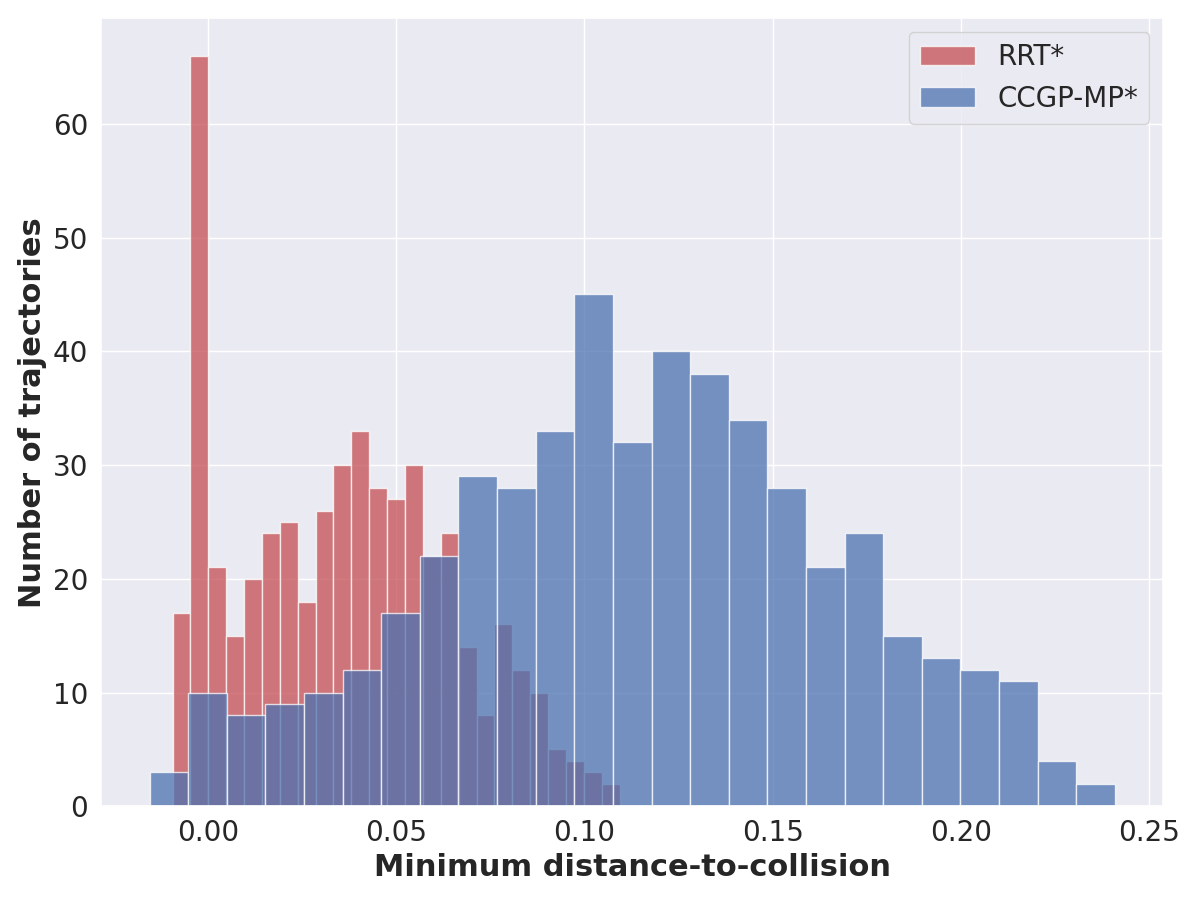}
    \end{subfigure}
    \caption{Left: The trajectories generated by RRT* (red) and CCGP-MP* (5\%) (blue) for a start (orange sphere) and goal (green sphere) pair in a real-world environment. Right: The histogram comparing the minimum distance-to-collision distribution for 500 trajectories for the adjacent path. As formulated, less than 5\% of the paths collide with an obstacle for CCGP-MP*.}
    \label{fig:gibson_test}
    \vspace{-0.5em}
\end{figure*}


For the Dubins Car model, we compared CCGP-MP* with RRT*. We did not consider the LQG-MP algorithm for this experiment since it does not explicitly solve the shortest path problem. As expected, CCGP-MP* has a much lower standard deviation of the success rate than RRT* (See Fig. \ref{fig:dubins_ccgp_comb} (right)). Fig. \ref{fig:dubins_ccgp_comb} reveals the reason for this improvement. The path planned by CCGP-MP* avoids the obstacles, even with the noisy robot model, while for the RRT* plan, the robot hits the obstacle. In Fig. \ref{fig:Paths}, we compare the paths generated by CCGP-MP* and RRT*, and like the Linear model, the CCGP-MP* deviates from the RRT* plan when chance-constraints are not met.

The results of the study on environment density on planning time and accuracy are summarized in Table \ref{tab:ccgp_mpTime}. The start and goal pairs were sampled from independent normal distributions with fixed means and 0.5 standard deviations. This distribution was fixed for all the environments. Apart from planning time, we also recorded the time taken by SHGO to find the global minimum for 50 random path segments. The GP model used for each environment had an equal number of support points, resulting in almost similar edge evaluation times. The overall planning time increases with the number of obstacles in space since the planner has to search more to find a feasible solution. We also observed that path accuracy was inversely proportional to the number of objects. The drop in accuracy could be attributed to the fact that more path segments are closer to the set threshold with a denser environment, thus decreasing the overall accuracy of the path.

\begin{table*}[h]   
    \caption{Study of number of obstacles on planning performance}
    \label{tab:ccgp_mpTime}
    \begin{center}
    \begin{tabular}{lcccccc}
        \toprule
        Number of Obstacles & 10 & 14 & 18 & 22 & 26 & 30  \\
        \midrule
        Edge Evaluation Time (sec) & 0.397 $\pm$ 0.058 & 0.394 $\pm$ 0.043 & 0.393 $\pm$ 0.061 & 0.385 $\pm$ 0.058 & 0.401 $\pm$ 0.061 & 0.406 $\pm$ 0.052 \\
        Planning Time (min) & 2.38$\pm$1.27 & 1.87 $\pm$ 1.13 & 2.39 $\pm$ 1.18 & 3.28 $\pm$ 2.07 & 4.70 $\pm$ 2.34 & 4.36 $\pm$ 2.59  \\
        Median Accuracy (\%) & 81.0 & 74.5 & 72.0 & 43.0 & 59.5 & 67.0\\
        \bottomrule
    \end{tabular}
    \end{center}
    \vspace{-1em}
\end{table*}
 
In addition to the simulated environment, we tested our planner on an indoor environment from the Gibson suite~\cite{xiazamirhe2018gibsonenv}. Fig. \ref{fig:gibson_test} (left) shows the trajectory generated by the RRT* and CCGP-MP* (5\%) for a single start and goal pair. Fig. \ref{fig:gibson_test} (right) shows the minimum distance to collision for each trajectory evaluated across 500 runs. For RRT, 16.4\% of the trajectories have the distance-to-collision less than zero, while for CCGP-MP, only 2.4\% of the trajectories have distance-to-collision less than zero. The value for CCGP-MP* also satisfies the delta threshold set for the planner, which is 0.05.

\section{CONCLUSION}
\noindent In this work, we introduced the Chance Constrained Gaussian Process Motion Planning, a chance-constrained motion planning approach that uses modeled distance-to-collision functions to plan in unstructured environments. Through the modeled distribution function, the planner guarantees that all states along the trajectory satisfy the given chance constraints. Simulation results on two robot systems showed that CCGP-MP and CCGP-MP* were able to generate paths that improved the planned path's success rate. 

One of the few limitations of our work lies in approximating distance-to-collision distribution as a Gaussian distribution. This simplification may not apply to some robotic systems. Another limitation centers around the scaling up of planning space. For larger maps, we require more support points to model our GP. For a large number of points($>$10000), the kernel matrix requires a considerable amount of memory~\cite[Chapter 8]{10.5555/1162254}, and the linear equations that need solving for inference becomes computationally intensive. One way to overcome these limitations is to use sparse GP models. One could even construct local GP models for larger maps similar to~\cite{9000569}.

There are multiple directions to be investigated further for the current work. One of them is extending the models to high DoF robotic systems. Rather than using the RBF kernel, the forward kernel (FK)~\cite{das2019forward}  would capture the distance function better. Another interesting avenue to investigate would be using a heteroscedastic GP to model the distance function, which might be more appropriate for time-varying robotic systems.  



\section*{APPENDIX}
\noindent For $x\in[0,1]$, the Taylor series expansion of $\frac{1}{(1-ax)^{\frac12}}$ about 0 is given by:
\begin{equation}
    \frac{1}{(1-ax)^{\frac12}} = 1 + \sum_{m=1}^{\infty}\frac{(ax)^m}{m!}\prod_{j=0}^{m-1} (\frac12+j)
    \label{eqn:Taylor_series}
\end{equation}
Using simple calculus we can show that:
\begin{align}
    1 + \sum_{m=1}^{\infty}\frac{a^m}{m!}\prod_{j=0}^{m-1} (\frac12+j) &= \frac{1}{(1-a)^{\frac12}} \label{eqn:series_1}\\
    \sum_{m=1}^{\infty}\frac{ma^m}{m!}\prod_{j=0}^{m-1} (\frac12+j) &= \frac{a}{2(1-a)^{\frac32}} \label{eqn:series_2}
\end{align}

\begin{lemma}
For $\bm{x_1}, \bm{x_2} \in [0, 1]^n$ and a postive definite symmetric matrix $M$ we have ,
\begin{equation}
    |\bm{x_1}^TM\bm{x_1} - \bm{x_2}^TM\bm{x_2}| \leq 2\lambda_{max}\sqrt{n}\|\bm{x_1}-\bm{x_2}\|
    \label{eqn:matrix_norm}
\end{equation}
where $\lambda_{max}$ is the maximum eigenvalue of $M$. 
\label{lemma:norm_ineq}
\end{lemma}
\begin{proof}
Since $M$ is symmetric, we can expand (\ref{eqn:matrix_norm}) as
\begin{align*}
    |\bm{x_1}^TM\bm{x_1} - \bm{x_2}^TM\bm{x_2}| &= |(\bm{x_1}-\bm{x_2})^TM(\bm{x_1}+\bm{x_2})| \qquad \quad\\ 
                            &\leq \lambda_{max}\|\bm{x_1}-\bm{x_2}\| \|\bm{x_1}+\bm{x_2}\| \\
                            &\leq 2\lambda_{max}\|\bm{x_1}-\bm{x_2}\| \sqrt{n}
\end{align*}
\end{proof}

\begin{lemma}
For $\bm{x_1}, \bm{x_2}\in [0, 1]^n$ , a symmetric positive definite matrix $M$ and $m\in \mathbb{N}$ we have, 
\begin{equation}
    |(\bm{x_1}^TM\bm{x_1})^m - (\bm{x_2}^TM\bm{x_2})^m| \leq \frac{2}{\sqrt{n}}m (n\lambda_{max})^m \|\bm{x_1}-\bm{x_2}\|
    \label{eqn:matrix_norm_m}
\end{equation}
where $\lambda_{max}$ is the maximum eigen value of $M$.
\label{lemma:matrix_norm_m}
\end{lemma}
\begingroup
\allowdisplaybreaks
\begin{proof}
The polynomial equation $x^m -y^m$ can be expressed as follows:
\begin{equation}
    x^m-y^m =(x-y)(x^{m-1}+x^{m-2}y + \ldots + y^{m-1}) 
\end{equation}
Let $\bm{x_i}^TM\bm{x_i}=\|\bm{x_i}\|_M^2$ for $i\in\{0,1\}$ we can express (\ref{eqn:matrix_norm_m}) in the same fashion as above,
\begin{IEEEeqnarray}{rl}
    |\|\bm{x_1}\|_M^{2^m} - \|\bm{x_2}&\|_M^{2^m}| = |(\|\bm{x_1}\|_M^2 - \|\bm{x_2}\|_M^2)(\|\bm{x_1}\|_M^{2^{(m-1)}}  \label{eqn:complex_exp} \\ &+\|\bm{x_1}\|_M^{2^{(m-2)}}\|\bm{x_2}\|_M^2+ \ldots +\|\bm{x_2}\|_M^{2^{(m-1)}})\| \nonumber
\end{IEEEeqnarray}
Since $\bm{x_i} \in[0,1]^n$ we can write $\|\bm{x_i}\|_M^2\leq \lambda_{max}n$ for $i\in\{0,1\}$, where $\lambda_{max}$ is the largest eigen value of $M$. Using this bound we can simplify (\ref{eqn:complex_exp}) as follows:
\begin{align*}
    (\ref{eqn:complex_exp})
                                    &\leq |\|\bm{x_1}\|_M^2 - \|\bm{x_2}\|_M^2| m (\lambda_{max}n)^{m-1} \qquad \qquad \qquad \quad \ \\ 
                                    &\leq 2\sqrt{n}\lambda_{max}\|\bm{x_1}-\bm{x_2}\| m (\lambda_{max}n)^{m-1} \\
                                    \tag{from Lemma \ref{lemma:norm_ineq}} \\
                                    & \leq \frac{2}{\sqrt{n}} m (\lambda_{max}n)^m\|\bm{x_1}-\bm{x_2}\|
\end{align*}
\end{proof}
\endgroup
\begin{lemma}
For $\bm{x_1}, \bm{x_2}\in[0, 1]^n$, a symmetric positive definite matrix $M$ and $m\in\mathbb{N}$ we have:
\begin{IEEEeqnarray}{rl}
    \|(\bm{x_1}^TM\bm{x_1})^m\bm{x_1} - &(\bm{x_2}^TM\bm{x_2})^m\bm{x_2}\| \leq     \label{eqn:norm_over_vec}
\\ &(1+2m)(\lambda_{max}n)^m\|\bm{x_1}-\bm{x_2}\|\nonumber
\end{IEEEeqnarray}
where $\lambda_{max}$ is the largest eigenvalue of $M$.

\begin{proof}
We can simplify (\ref{eqn:norm_over_vec}) using Lemma \ref{lemma:matrix_norm_m}.
\begin{align*}
    &\|\ \|\bm{x_1}\|_M^{2m}\bm{x_1} - \|\bm{x_2}\|_M^{2m}\bm{x_2}\ \| =\\ 
    &\qquad \|\ \|\bm{x_1}\|_M^{2m}(\bm{x_1}-\bm{x_2}) +(\|\bm{x_1}\|_M^{2m}-\|\bm{x_2}\|_M^{2m})\bm{x_2}\ \| \\
    &\ \leq \|\ \|\bm{x_1}\|_M^{2m}(\bm{x_1}-\bm{x_2})\ \| + \|\ (\|\bm{x_1}\|_M^{2m}-\|\bm{x_2}\|_M^{2m})\bm{x_2}\ \| \\
    &\  \leq (\lambda_{max}n)^m\|\bm{x_1}-\bm{x_2}\| + \sqrt{n}\frac{2}{\sqrt{n}}m(\lambda_{max}n)^m\|\bm{x_1}-\bm{x_2}\| \\
                                             \tag{from Lemma \ref{lemma:matrix_norm_m}} \\
    & \leq (1+2m)(\lambda_{max}n)^m\|\bm{x_1}-\bm{x_2}\|
\end{align*}
\end{proof}
\label{lemma:norm_over_vec}
\end{lemma}

\vspace{-1em}
\section*{ACKNOWLEDGMENT}
We thank Vikas Dhiman, Florian Richter, and Nikhil Das for insightful discussions.
\vspace{-0.5em}

\bibliographystyle{IEEEtran}
\bibliography{main}
\end{document}